\def\year{2017}\relax
\documentclass[letterpaper]{article}
\usepackage{aaai17}
\usepackage{times}
\usepackage{helvet}
\usepackage{courier}
\usepackage{graphicx}
\usepackage{amsmath}
\usepackage{amssymb}
\usepackage{algorithm}
\usepackage{multirow} 
\usepackage{algorithmic}
\usepackage{booktabs,caption,fixltx2e}
\usepackage[flushleft]{threeparttable}
\usepackage{subfig}
\usepackage{amsthm}
\usepackage{threeparttable}
\usepackage{placeins}
\renewcommand{\vec}[1]{\mathbf{#1}}
\newtheorem{theorem}{Theorem}
\newtheorem{corollary}{Corollary}

\begin{document}
%
\title{Twin Learning for Similarity and Clustering: A Unified Kernel Approach}
\author{Zhao Kang, Chong Peng, Qiang Cheng\\
Department of Computer Science, Southern Illinois University, Carbondale, IL 62901, USA\\
\{Zhao.Kang, pchong, qcheng\}@siu.edu\\
}
\maketitle
\begin{abstract}
Many similarity-based clustering methods work in two separate steps including similarity matrix computation and subsequent spectral clustering. However, similarity measurement is challenging because it is usually impacted by many factors, e.g., the choice of similarity metric, neighborhood size, scale of data, noise and outliers. Thus the learned similarity matrix is often not suitable, let alone optimal, for the subsequent clustering. In addition, nonlinear similarity often exists in many real world data which, however, has not been effectively considered by most existing methods. To tackle these two challenges, we propose a model to simultaneously learn cluster indicator matrix and similarity 
information in kernel spaces in a principled way. We show theoretical relationships to kernel k-means, k-means, and spectral clustering methods. Then, to address the practical issue of how to select the most suitable kernel for a particular clustering task, we further extend our model with a multiple kernel learning ability. With this joint model, we can automatically accomplish three subtasks of finding the best cluster indicator matrix, the most accurate similarity relations and the optimal combination of multiple kernels. By leveraging the interactions between these three subtasks in a joint framework, each subtask can be iteratively boosted by using the results of the others towards an overall optimal solution. Extensive experiments are performed to  demonstrate the effectiveness of our method.
\end{abstract}

\section{Introduction}
Clustering is a fundamental topic in data mining and machine learning  \cite{pengnonnegative}. It partitions data points into different groups, such that the objects within a group are similar to one another and different from those in other groups. Various methods have been proposed over the past decades. Some well-known algorithms include k-means clustering \cite{macqueen1967some}, spectral clustering \cite{ng2002spectral}, and hierarchical clustering \cite{johnson1967hierarchical}. 

Thanks to the simplicity and the effectiveness, the k-means algorithm is widely used. However, it fails to identify arbitrarily shaped clusters. Kernel k-means \cite{scholkopf1998nonlinear} has been developed to capture nonlinear structure information hidden in data sets. Kernel-based learning methods requires one to specify a kernel, which means one assumes a certain shape of the underlying data space. Thus the performance of kernel-based methods are largely affected by the choice of kernel.  

Spectral clustering does a low-dimensional embedding of the similarity matrix of the data before performing k-means clustering \cite{ng2002spectral}. The similarity between every pair of points, as an input, leverages the manifold information in this clustering model. Thus similarity-based clustering methods usually show better performance than k-means algorithm. However, the performance of this kind of methods is largely determined by the similarity matrix \cite{huang2015new}. Any variations during the similarity measurement, such as metric, neighborhood size, and data scale, may lead to suboptimal performance.

Recently, self-expression has been successfully utilized in subspace recovery \cite{elhamifar2009sparse,luo2011multi}, low rank representation \cite{kang2015cikm,kang2015robust}, and recommender systems \cite{kang2016top}. It represents each data point in terms of the other points. By solving an optimization problem, the similarity information is automatically learned from the data. This approach can not only reveal low-dimensional structure, but also be robust to noise and data scale \cite{huang2015new}.

\section{Contributions}
In this paper, we perform clustering built upon the idea of using samples from the data to ``express itself". Rather than local structure learning \cite{nie2014clustering}, this approach extracts the global structure of data and can be extended to kernel spaces. Unlike existing clustering algorithms that work in two separate steps, we simultaneously learn similarity matrix and cluster indicators by imposing a rank constraint on the Laplacian matrix of the learned similarity matrix. By leveraging the intrinsic interactions between learning similarity and cluster indicators, our proposed model seamlessly integrates them into a joint framework, where the result of one task is used to improve the other one. To capture the nonlinear structure information inherent in many real world data sets, we directly develop our method in a kernel space, which is well known for its ability to explore the nonlinear relation. We design an efficient algorithm to find an optimal solution to our model, and show the theoretical analysis on the connections to kernel k-means, k-means, and spectral clustering methods. 

While effective, the kernel in use often has enormous influence on the performance of any kernel method. Unfortunately, the most suitable kernel for a specific task is usually unknown in advance. Exhaustive search on a user-defined pool of kernels is time-consuming and impractical when the sizes of the pool and data become large \cite{zeng2011feature}. Thus we further propose a multiple kernel algorithm for our model. Another benefit of applying multiple kernels is that we can fully utilize information from different sources equipped with heterogeneous features \cite{yu2012optimized}. To alleviate the effort for kernel construction and integrating complementary information, we learn an appropriate consensus kernel from a linear combination of multiple input kernels. As a result, our joint model can simultaneously learn the similarity information, cluster indicator matrix, and the optimal combination of multiple kernels. Extensive empirical results on real-world benchmark data sets show that our method consistently outperforms other state-of-the-art methods. 

\textbf{Notations.} In this paper, matrices are written as upper case letters and vectors are represented by boldface lower-case letters. The $i$-th column and the $(i,j)$-th element of matrix $X$ are denoted by $X_i$ and $x_{ij}$, respectively. The $\ell_2$-norm of a vector $\vec{x}$ is defined as $\|\vec{x}\|^2=\vec{x}^T\cdot\vec{x}$, where $T$ means transpose. $I$ denotes the identity matrix and $\vec{1}$ denotes a column vector with all the elements as one. Tr($\cdot$)  is the trace operator. $0\leq Z\leq 1$ means all elements of $Z$ are in the range $[0,1]$.

\section{Clustering with Single Kernel}
According to the self-expressive property \cite{elhamifar2009sparse}, 
\begin{equation}
\label{self}
X_i\approx \sum_{j}X_jz_{ij},\quad
s.t. \quad Z_i^T\vec{1}=1, 0\leq z_{ij}\leq 1,
\end{equation}
 where $z_{ji}$ is the weight for $j$-th sample.
More similar data points should receive bigger weights and the weights should be smaller for less similar points. Thus $Z$ is also called similarity matrix, which represents the global structure of data. Note that (\ref{self}) is in a similar spirit of Locally Linear Embedding (LLE) \cite{roweis2000nonlinear}, which assumes that the data points lie on a manifold and each data point can be expressed as a linear combination of its nearest neighbors. The difference from LLE lies in the fact that we specify no neighborhood, which is automatically determined by our method. 

To obtain $Z$, we solve the following problem:
 \begin{equation}
\label{sparse}
\begin{split}
\min_Z \|X- XZ\|_F^2+\alpha \|Z\|_F^2,\\
s.t. \quad  Z^T\vec{1}=\vec{1}, \quad 0\leq Z\leq 1,
\end{split}
\end{equation} 
where the first term is to measure reconstruction error, the second term is imposed to avoid the trivial solution $Z=I$, and $\alpha$ is a trade-off parameter.

One drawback of (\ref{sparse}) is that it assumes linear relations between samples. To recover the nonlinear relations between the data points, we extend (\ref{sparse}) to kernel spaces by deploying a general kernelization framework \cite{zhang2010general}. Define $\phi: \mathcal{R}^D\rightarrow \mathcal{H}$ to be a kernel mapping the data samples from the input space to a reproducing kernel Hilbert space $\mathcal{H}$. For $X=[X_1,\cdots,X_n]$ containing $n$ samples, the transformation is  $\phi(X)=[\phi(X_1),\cdots,\phi(X_n)]$. The kernel similarity between data samples $X_i $ and $X_j$ is defined through a predefined kernel as  $K_{X_i,X_j}=<\phi(X_i),\phi(X_j)>$. It is easy to observe that all similarities can be computed exclusively using the kernel function and one does not need to know the transformation $\phi$. This is known as the kernel trick and it greatly simplifies the computations in the kernel space when the kernels are precomputed.
Then (\ref{sparse}) becomes
\begin{equation}
\begin{split}
&\min_Z Tr(K-2KZ+Z^TKZ)+\alpha \|Z\|_F^2  \\
&s.t.\quad  Z^T\vec{1}=\vec{1}, \quad 0\leq Z\leq 1.
\end{split}
\label{our}
\end{equation}
By solving above problem, we learn the linear sparse relations of $\phi(X)$, and thus the nonlinear relations among $X$. Note that (\ref{our}) goes back to (\ref{sparse}) if a linear kernel is adopted.

Ideally, we expect that the number of connected components in $Z$ are exactly $c$ if the given data set $X$ consists of $c$ clusters (that is, $Z$ is block diagonal with proper permutations). However, the solution $Z$ from (\ref{our}) might not satisfy this desired property. Therefore, we will add another constraint based on the following theorem \cite{mohar1991laplacian}. 
\begin{theorem}
The multiplicity $c$ of the eigenvalue 0 of the Laplacian matrix $L$ of $Z$ is equal to the number of connected components in the graph with the similarity matrix $Z$.
\end{theorem}
Theorem 1 means that $rank(L)=n-c$ if the similarity matrix $Z$ contains exactly $c$ connected components. Thus our new clustering model is to solve:
\begin{equation}
\begin{split}
&\min_Z Tr(K-2KZ+Z^TKZ)+\alpha \|Z\|_F^2  \\
&s.t.\quad  Z^T\vec{1}=\vec{1}, \quad 0\leq Z\leq 1,\quad rank(L)=n-c.
\end{split}
\label{ournew}
\end{equation}

Problem (\ref{ournew}) is not easy to tackle, since $L:=D-\frac{Z^T+Z}{2}$, where $D\in \mathcal{R}^{n\times n}$ is a diagonal matrix with the $i$-th diagonal element $\sum_j\frac{z_{ij}+z_{ji}}{2}$, also depends on similarity matrix $Z$. 

Here $L$ is positive semi-definite, thus $\sigma_i(L)\geq 0$, where $\sigma_i(L)$ is the $i$-th smallest eigenvalue of $L$. $rank(L)=n-c$ is equivalent to $\sum_{i=1}^{c} \sigma_i(L)=0$. It is not easy to enforce this constraint because the optimization problem with a rank constraint is known to be of combinatorial complexity. To mitigate the difficulty, \cite{wang2015discriminative,nie2016constrained} incorporates the rank constraint into the objective function as a regularizer. Motivated by this consideration, we relax the constraint and reformulate our model as
\begin{equation}
\begin{split}
&\min_Z Tr(K-2KZ+Z^TKZ)+\alpha \|Z\|_F^2+\beta \sum_{i=1}^{c} \sigma_i(L) \\
&s.t.\quad  Z^T\vec{1}=\vec{1}, \quad 0\leq Z\leq 1.
\end{split}
\label{newmodel}
\end{equation}
The minimization  will make the regularizer $\sum_{i=1}^{c} \sigma_i(L)\rightarrow 0$ if $\beta$ is large enough. Then the constraint $rank(L)=n-c$ will be satisfied. 

Problem (\ref{newmodel}) is still a challenging problem because of the last term. Fortunately, it can be solved by using Ky Fan's Theorem \cite{fan1949theorem}, i.e., 
\begin{equation}
\sum_{i=1}^{c} \sigma_i(L)=\min_{P^TP=I} Tr(P^TLP),
\end{equation}
where $P\in\mathcal{R}^{n\times c}$ is the indicator matrix. The $c$ elements of $i$-th row $P_{i,:}\in \mathcal{R}^{c\times 1}$ are the measure of the membership of data point $X_i$ belonging to the $c$ clusters. Finally, our model of twin learning for similarity and clustering with a single kernel (SCSK) is formulated as
\begin{equation}
\begin{split}
&\min_{Z, P} Tr(K-2KZ+Z^TKZ)+\alpha \|Z\|_F^2+\beta Tr(P^TLP), \\
& s.t.\quad Z^T\vec{1}=\vec{1}, \quad 0\leq Z\leq 1, \quad P^TP=I.
\end{split}
\label{finalmodel}
\end{equation}
By solving (\ref{finalmodel}), we directly obtain the indicator matrix $P$; therefore, we do not need to perform spectral clustering any more. By alternatively updating $Z$ and $P$, they can improve each other and optimize (\ref{finalmodel}). 

\subsection{Optimization Algorithm} 
We use an alternating optimization strategy for (\ref{finalmodel}). When $Z$ is fixed, (\ref{finalmodel}) becomes
\begin{equation}
\min_{P^TP=I} Tr(P^TLP).
\end{equation}
The optimal solution $P$ is obtained by the $c$ eigenvectors of $L$ corresponding to the $c$ smallest eigenvalues.

When $P$ is fixed, (\ref{finalmodel}) can be reformulated column-wisely as:
\begin{equation}
\begin{split}
&\min_{Z_i} K_{ii}-2K_{i,:}Z_i+Z_i^TKZ_i+\alpha Z_i^TZ_i+\frac{\beta}{2}d_i^T Z_i,\\
&s.t. \quad Z_i^T\vec{1}=1,\quad 0\leq z_{ij}\leq 1,
\label{solvez}
\end{split}
\end{equation}
where $d_i\in\mathcal{R}^{n\times 1}$ is a vector with the $j$-th element $d_{ij}$ being $d_{ij}=\|P_{i,:}-P_{j,:}\|^2$. To obtain (\ref{solvez}), the important equation in spectral analysis 
\begin{equation}
\sum_{i,j}\frac{1}{2}\|P_{i,:}-P_{j,:}\|^2z_{ij}=Tr(P^TLP)
\end{equation}
is used.
(\ref{solvez}) can be further simplified as
\begin{equation}
\begin{split}
&\min_{Z_i} Z_i^T(\alpha I+K)Z_i+(\frac{\beta d_i^T}{2}-2K_{i,:})Z_i,\\
&s.t. \quad Z_i^T\vec{1}=1,\quad 0\leq z_{ij}\leq 1.
\end{split}
\label{solveZ}
\end{equation}
This problem can be solved by many existing quadratic programing packages. The complete algorithm is outlined in Algorithm 1. 
\begin{algorithm}[!tb]
\scriptsize
\caption{The algorithm of SCSK }
\label{alg1}
 {\bfseries Input:} Kernel matrix $K$, parameters  $\alpha>0$,  $\beta>0$.\\
{\bfseries Initialize:} Random matrix $Z$.\\
 {\bfseries REPEAT}
\begin{algorithmic}[1]
 \STATE Update $P$, which is formed by the $c$ eigenvectors of $L=D-\frac{Z^T+Z}{2}$ corresponding to the $c$ smallest eigenvalues.

   \STATE For each $i$, update the $i$-th column of $Z$ by solving problem (\ref{solveZ}).

\end{algorithmic}
\textbf{ UNTIL} {stopping criterion is met.}
\end{algorithm}
\section{Theoretical Analysis of SCSK Model}
In this section, we present a theoretical analysis of SCSK and its connections to kernel k-means, k-means, and SC. 
\subsection{Connection to Kernel K-means and K-means}
Here we introduce a theorem which states the equivalence of SCSK and kernel k-means, k-means under some condition.
\begin{theorem}
When $\alpha\to\infty$, the problem (\ref{ournew}) is equivalent to kernel k-means problem.
\end{theorem}
\begin{proof}
The constraint $rank(L)=n-c$ in (\ref{ournew}) makes the solution $Z$ block diagonal. Let $Z^i\in\mathcal{R}^{n_i\times n_i}$ denote the similarity matrix of the $i$-th component of $Z$, where $n_i$ is the number of data samples in the component. Problem (\ref{ournew}) is equivalent to solving the following problem for each $i$:
\begin{equation}
\begin{split}
&\min_{Z^i}  \|\phi(X^i)-\phi(X^i)Z^i\|_F^2+\alpha \|Z^i\|_F^2  \\
&s.t.\quad  (Z^i)^T\vec{1}=\vec{1}, \quad 0\leq Z^i\leq 1,
\end{split}
\label{equal}
\end{equation}
where $X^i$ consists of the samples corresponding to the $i$-th component of $Z$.
When $\alpha\to\infty$, the above problem becomes
\begin{equation}
\begin{split}
&\min_{Z^i} \alpha \|Z^i\|_F^2  \\
s.t.&\hspace{.2cm}  (Z^i)^T\vec{1}=\vec{1}, \hspace{.1cm} 0\leq Z^i\leq 1.
\end{split}
\end{equation}
The optimal solution is that all elements of $Z^i$ are equal to $\frac{1}{n_i}$. 

Thus when $\alpha\to\infty$, the optimal solution $Z$ to problem (\ref{ournew}) is 
\begin{equation}
    z_{ij}=
    \begin{cases}
      \frac{1}{n_k}, & \text{if $X_i$ and $X_j$ are in the same $k$-th component} \\
      0, & \text{otherwise}
    \end{cases}
  \end{equation}
Denote the solution set of this form by $ \mathcal{C}$. It is easy to observe that $\|Z\|_F^2=c$. Thus (\ref{ournew}) becomes
\begin{equation}
\min_{Z_i\in \mathcal{C}} \sum_i\|\phi(X_i)-\phi(X)Z_i\|^2
\label{kernels}
\end{equation}
It is easy to deduce that $\phi(X)Z_i$ is the mean of cluster $c_i$ in the kernel space. Therefore, (\ref{kernels}) is exactly the kernel k-means. Thus our proposed algorithm is to solve the kernel k-means problem when $\alpha\to\infty$.
\end{proof}
\begin{corollary}
When $\alpha\to\infty$ and a linear kernel is adopted, the problem (\ref{ournew}) is equivalent to k-means problem.
\end{corollary}
\begin{proof}
It is obvious when one does not use any transformations on $X$ in (\ref{kernels}).
\end{proof}
\subsection{Connection to Spectral Clustering}
With a predefined similarity $Z$, spectral clustering is to solve the following problem:
\begin{equation}
\min_{P^TP=I} Tr(P^TLP).
\end{equation}
The optimal solution $P$ is obtained by the $c$ eigenvectors of $L$ corresponding to the $c$ smallest eigenvalues. Generally, $P$ can not be directly used for clustering since $Z$ does not have exactly $c$ connected components. To obtain the final clustering results, k-means or some other discretization procedures must be performed on $P$ \cite{huang2013spectral}. 

In our proposed algorithm, the similarity matrix $Z$ is not predefined as the existing spectral clustering methods in the literature. Also, the similarity matrix $Z$ is learned by taking account of the clustering task at hand, as opposed to the existing subspace clustering methods in the literature which only focus on learning the similarity matrix  
 $Z$ without considering the effect of clustering on $Z$ \cite{peng2015subspace}. In our method, the graph with the learned $Z$ will be partitioned into $c$ connected components by using $P$. The optimal solution $P$ is formed by the $c$ eigenvectors of $L$, which is defined by $Z$, corresponding to the $c$ smallest eigenvalues. Therefore, the proposed algorithm learns the similarity matrix $Z$ and the cluster indicator matrix $P$ simultaneously in a coupled way, which leads to a better result in real applications than existing spectral methods as shown in our experiments, since it learns an adaptive graph for clustering.

\section{Clustering with Multiple Kernels}
Although model (\ref{finalmodel}) can automatically learn the similarity matrix  and cluster indicator matrix, its performance will largely be determined by the choice of kernel. It is often impractical to exhaustively search for the most suitable kernel. Moreover, real world data sets are often generated from different sources along with heterogeneous features. Single kernel method may not be able to fully utilize such information. Multiple kernel learning is capable of integrating complementary information and identifying a suitable kernel for a given task. Here we present a way to learn an appropriate consensus kernel from a convex combination of several predefined kernel matrices.

Suppose there are a total number of $r$ different kernel functions $\{K^i\}_{i=1}^{r}$. Correspondingly, there would be $r$ different kernel spaces denoted as $\{\mathcal{H}^i\}_{i=1}^{r}$. An augmented Hilbert space, $\tilde{\mathcal{H}}=\bigoplus_{i=1}^r \mathcal{H}^i$, can be constructed by concatenating all kernel spaces and by using the mapping of $\tilde{\phi}(\vec{x})=[\sqrt{w_1}\phi_1(\vec{x}),$ $\sqrt{w_2}\phi_2(\vec{x}),...,\sqrt{w_r}\phi_r(\vec{x})]^T$ with different weights $\sqrt{w_i}(w_i\ge 0)$. Then the combined kernel $K_w$ can be represented as \cite{zeng2011feature}
\begin{equation}
\label{ukernel}
K_w(\vec{x},\vec{y})=<\phi_w(\vec{x}),\phi_w(\vec{y})>=\sum\limits_{i=1}^r w_iK^i(\vec{x},\vec{y}).
\end{equation} 
 Note that the convex combination of the positive semi-definite kernel matrices  $\{K^i\}_{i=1}^{r}$ is still a positive semi-definite kernel matrix. Thus the combined kernel still satisfies Mercer's condition. Then we propose our joint similarity learning and clustering with multiple kernel (SCMK) model which can be written as
\begin{equation}
\begin{split}
&\min_{Z, P, \vec{w}}\! Tr(\!K_w\!-\!2K_wZ\!+\!Z^TK_wZ\!)\!+\!\alpha \|Z\|_F^2\!+\!\beta\! Tr(P^TLP), \\
& s.t.\quad Z^T\vec{1}=\vec{1}, \quad 0\leq Z\leq 1, \quad P^TP=I, \\
&\hspace{.8cm}K_w=\sum\limits_{i=1}^r w_iK^i, \sum\limits_{i=1}^r \sqrt{w_i}=1, w_i\ge 0.
\end{split}
\label{multimodel}
\end{equation}

By iteratively updating $Z, P, \vec{w}$, each of them will be adaptively refined according to the results of the other two.
\subsection{Optimization }
Problem (\ref{multimodel}) can be solved by alternatively updating $Z$, $P$, and $\vec{w}$, while holding the other variables as constant.

1) Optimizing with respect to $Z$ and $P$ when $\vec{w}$ is fixed: We can directly calculate $K_w$, and the optimization problem is exactly (\ref{finalmodel}). Thus we just need to use Algorithm 1 with $K_w$ as the input kernel matrix.

2) Optimizing with respect to $\vec{w}$ when $Z$ and $P$ are fixed: Solving (\ref{multimodel}) with respect to $\vec{w}$ can be rewritten as \cite{cai2013heterogeneous}
\begin{equation}
\label{optie}
\min_\vec{w} \sum\limits_{i=1}^r w_i h_i   \quad s.t.\quad  \sum\limits_{i=1}^r \sqrt{w_i}=1, \quad w_i\ge 0, 
\end{equation}
where 
\begin{equation}
\label{h}
h_i=Tr(K^i-2K^iZ+Z^TK^iZ).
\end{equation}
The Lagrange function of (\ref{optie}) is 
\begin{equation}
\mathcal{J}(\vec{w})=\vec{w}^T\vec{h}+\gamma (1-\sum_{i=1}^r\sqrt{w_i}).
\end{equation}
By utilizing the Karush-Kuhn-Tucker (KKT) condition with $\frac{\partial \mathcal{J}(\vec{w})}{\partial w_i}=0$ and the constraint $\sum\limits_{i=1}^r \sqrt{w_i}=1$, we obtain the solution of $\vec{w}$ as follows:
\begin{equation}
\label{weight}
w_i=(h_i \sum_{j=1}^r \frac{1}{h_j})^{-2}.
\end{equation}
In Algorithm 2 we provide a complete algorithm for solving (\ref{multimodel}).

\begin{algorithm}
 \scriptsize
\caption{The algorithm of SCMK}
\label{alg2}
 {\bfseries Input:} A set of kernel matrices $\{K^i\}_{i=1}^r$, parameters  $\alpha>0$,  $\beta>0$.\\
{\bfseries Initialize:} Random matrix $Z$, $w_i=1/r$.\\
 {\bfseries REPEAT}
\begin{algorithmic}[1]
\STATE Calculate $K_\vec{w}$ by (\ref{ukernel}).
 \STATE Update $P$ with the $c$ smallest eigenvectors of $L=D-\frac{Z^T+Z}{2}$.
\STATE For each $i$, update the $i$-th column of $Z$ by (\ref{solveZ}).
\STATE Calculate $\vec{h}$ by (\ref{h}).
\STATE Update $\vec{w}$ by (\ref{weight}).
\end{algorithmic}
\textbf{ UNTIL} {stopping criterion is met.}
\end{algorithm}

\section{Experiments}
\captionsetup{position=top}
\begin{table}[!htbp]
\centering
\caption{Description of the data sets}
\label{data}
\renewcommand{\arraystretch}{0.7}
\begin{tabular}{|l|c|c|c|}
\hline
&\textrm{\# instances}&\textrm{\# features}&\textrm{\# classes}\\\hline
\textrm{YALE}&165&1024&15\\\hline
\textrm{JAFFE}&213&676&10\\\hline
\textrm{ORL}&400&1024&40\\\hline
\textrm{AR}&840&768&120\\\hline
\textrm{BA}&1404&320&36\\\hline
\textrm{TR11}&414&6429&9\\\hline
\textrm{TR41}&878&7454&10\\\hline
\textrm{TR45}&690&8261&10\\\hline
\end{tabular}
\end{table}

%

\captionsetup{position=top}
\begin{table*}[!ht]
\centering
\tiny
\renewcommand{\arraystretch}{0.5}
\setlength{\tabcolsep}{.02pt}
\subfloat[Accuracy(\%)\label{acc}]{
\resizebox{1.04\textwidth}{!}{
\begin{tabular}{ |l  |c |c| c| c| c| c| c| c | |c| c| c| c| c }
\hline
Data 	& KKM	& KKM-a	& SC	& SC-a	& RKKM	& RKKM-a & SCSK			  & SCSK-a & MKKM  & AASC  & RMKKM & SCMK 			\\	\hline
YALE  	& 47.12	& 38.97	& 49.42	& 40.52	& 48.09	& 39.71	 & \textbf{55.85} & 45.35  & 45.70 & 40.64 & 52.18 & \textbf{56.97} \\	\hline
JAFFE 	& 74.39 & 67.09 & 74.88 & 54.03 & 75.61 & 67.98	 & \textbf{99.83} & 86.64  & 74.55 & 30.35 & 87.07 & \textbf{100.00}\\	\hline
ORL   	& 53.53 & 45.93 & 57.96 & 46.65 & 54.96 & 46.88  & \textbf{62.35} & 50.50  & 47.51 & 27.20 & 55.60 & \textbf{65.25}	\\	\hline
AR   	& 33.02 & 30.89 & 28.83 & 22.22 & 33.43 & 31.20  & \textbf{56.79} & 41.35  & 28.61 & 33.23 & 34.37 & \textbf{62.38}	\\	\hline
BA 		& 41.20 & 33.66 & 31.07 & 26.25 & 42.17 & 34.35  & \textbf{47.72} & 39.50  & 40.52 & 27.07 & 43.42 & \textbf{47.34}	\\	\hline
TR11   	& 51.91 & 44.65 & 50.98 & 43.32 & 53.03 & 45.04  & \textbf{71.26} & 54.79  & 50.13 & 47.15 & 57.71 & \textbf{73.43}	\\	\hline
TR41  	& 55.64 & 46.34 & 63.52 & 44.80 & 56.76 & 46.80  & \textbf{67.43} & 53.13  & 56.10 & 45.90 & 62.65 & \textbf{67.31}	\\	\hline
TR45  	& 58.79 & 45.58 & 57.39 & 45.96 & 58.13 & 45.69  & \textbf{74.02} & 53.38  & 58.46 & 52.64 & 64.00 & \textbf{74.35}	\\	\hline
\end{tabular}
}

}\\
\renewcommand{\arraystretch}{0.5}
\subfloat[NMI(\%)\label{NMI}]{
\resizebox{1.04\textwidth}{!}{
\begin{tabular}{ |l  |c |c| c| c| c| c| c| c | |c| c| c| c| c }
\hline
Data 	& KKM	& KKM-a	& SC	& SC-a	& RKKM	& RKKM-a & SCSK			  & SCSK-a& MKKM  & AASC  & RMKKM & SCMK 			\\	\hline
YALE 	& 51.34 & 42.07 & 52.92 & 44.79 & 52.29 & 42.87  & \textbf{56.50} & 45.07 & 50.06 & 46.83 & 55.58 & \textbf{56.52}	\\	\hline
JAFFE 	& 80.13 & 71.48 & 82.08 & 59.35 & 83.47 & 74.01  & \textbf{99.35} & 84.67 & 79.79 & 27.22 & 89.37 & \textbf{100.00}	\\	\hline
ORL 	& 73.43 & 63.36 & 75.16 & 66.74 & 74.23 & 63.91  & \textbf{78.96} & 63.55 & 68.86 & 43.77 & 74.83 & \textbf{80.04}	\\	\hline
AR 		& 65.21 & 60.64 & 58.37 & 56.05 & 65.44 & 60.81  & \textbf{76.02} & 59.70 & 59.17 & 65.06 & 65.49 & \textbf{81.51}	\\	\hline
BA  	& 57.25 & 46.49 & 50.76 & 40.09 & 57.82 & 46.91  & \textbf{63.04} & 52.17 & 56.88 & 42.34 & 58.47 & \textbf{62.94}	\\	\hline
TR11  	& 48.88 & 33.22 & 43.11 & 31.39 & 49.69 & 33.48  & \textbf{58.60} & 37.58 & 44.56 & 39.39 & 56.08 & \textbf{60.15}	\\	\hline
TR41 	& 59.88 & 40.37 & 61.33 & 36.60 & 60.77 & 40.86  & \textbf{65.50} & 43.18 & 57.75 & 43.05 & 63.47 & \textbf{65.11}	\\	\hline
TR45	& 57.87 & 38.69 & 48.03 & 33.22 & 57.86 & 38.96  & \textbf{74.24} & 44.36 & 56.17 & 41.94 & 62.73 & \textbf{74.97}	\\	\hline
\end{tabular}
}
}\\
\renewcommand{\arraystretch}{0.5}
\subfloat[ Purity(\%)\label{purity}]{
\resizebox{1.04\textwidth}{!}{
\begin{tabular}{|l  |c |c| c| c| c| c| c| c | |c| c| c| c| c}\hline
Data 	& KKM	& KKM-a	& SC	& SC-a	& RKKM			& RKKM-a & SCSK			  & SCSK-a& MKKM  & AASC  & RMKKM & SCMK 			\\	\hline
YALE   	& 49.15 & 41.12 & 51.61 & 43.06 & 49.79 		& 41.74  & \textbf{57.27} & 55.79 & 47.52 & 42.33 & 53.64 & \textbf{60.00}	\\	\hline
JAFFE  	& 77.32 & 70.13 & 76.83 & 56.56 & 79.58 		& 71.82  & \textbf{99.85} & 96.53 & 76.83 & 33.08 & 88.90 & \textbf{100.00}	\\	\hline
ORL   	& 58.03 & 50.42 & 61.45 & 51.20 & 59.60 		& 51.46  & \textbf{74.00} & 70.37 & 52.85 & 31.56 & 60.23 & \textbf{77.00}	\\	\hline
AR   	& 35.52 & 33.64 & 33.24 & 25.99 & 35.87 		& 33.88  & \textbf{63.45} & 62.37 & 30.46 & 34.98 & 36.78 & \textbf{82.62} 	\\	\hline
BA  	& 44.20 & 36.06 & 34.50 & 29.07 & 45.28 		& 36.86  & \textbf{52.36} & 49.79 & 43.47 & 30.29 & 46.27 & \textbf{52.12}		\\	\hline
TR11   	& 67.57 & 56.32 & 58.79 & 50.23 & 67.93 		& 56.40  & \textbf{82.85} & 80.76 & 65.48 & 54.67 & 72.93 & \textbf{87.44}	\\	\hline
TR41  	& 74.46 & 60.00 & 73.68 & 56.45 & \textbf{74.99}& 60.21  & 73.23 		  & 71.21 & 72.83 & 62.05 & \textbf{77.57} & 73.69	\\	\hline
TR45  	& 68.49 & 53.64 & 61.25 & 50.02 & 68.18 		& 53.75  & \textbf{78.26} & 77.76 & 69.14 & 57.49 & 75.20 & \textbf{78.26}	\\	\hline
\end{tabular}
}}
\caption{Clustering results measured on benchmark data sets. '-a' denotes the average performance on those 12 kernels. Both the best results for single kernel and multiple kernel methods are highlighted in boldface. \label{clusterres}}
\end{table*}
\captionsetup{position=bottom}
\begin{figure*}[hbtp]
\centering
\subfloat[Acc\label{acc}]{\includegraphics[width=.3\textwidth]{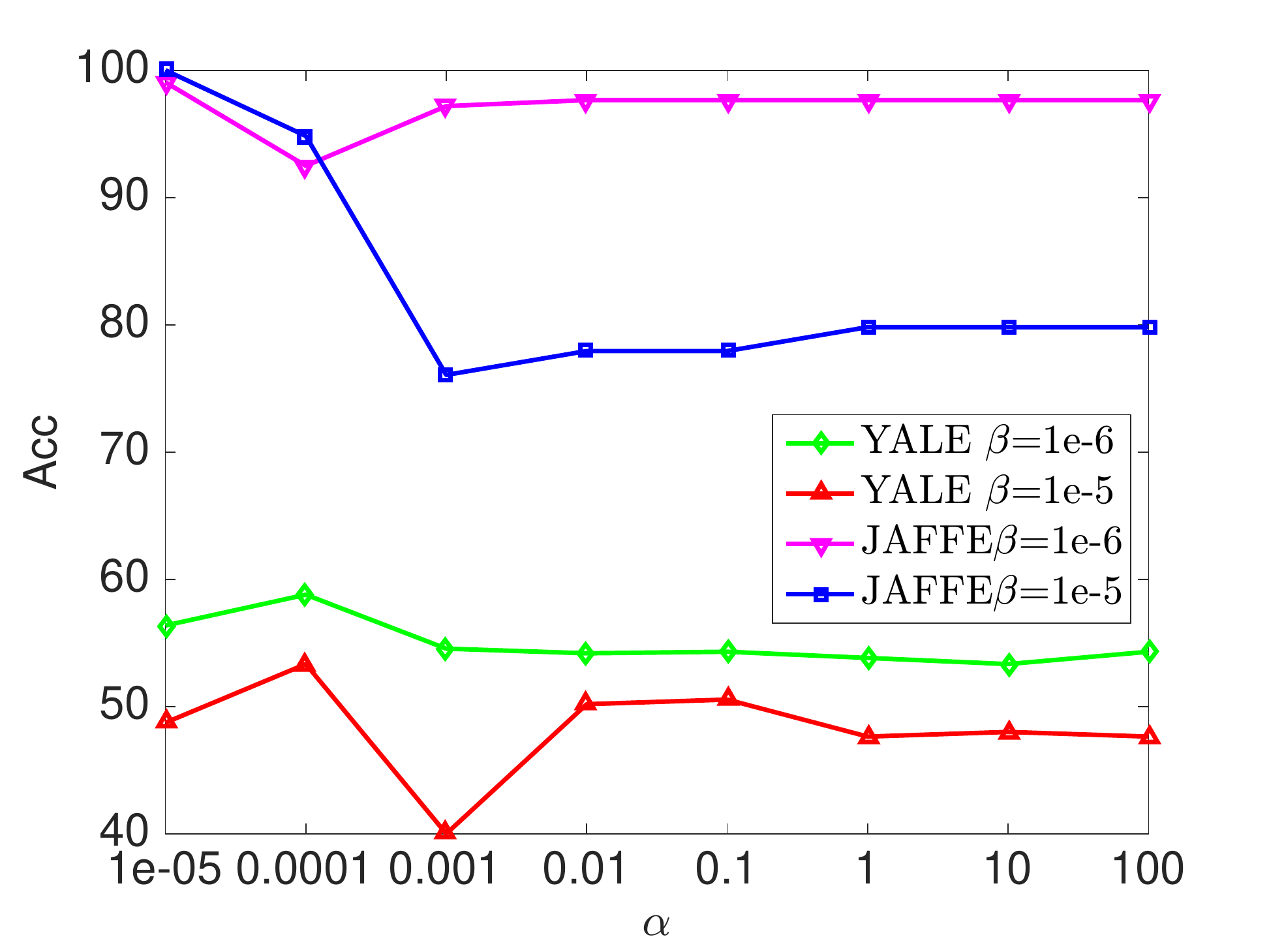}}
\subfloat[NMI\label{nmi}]{\includegraphics[width=.3\textwidth]{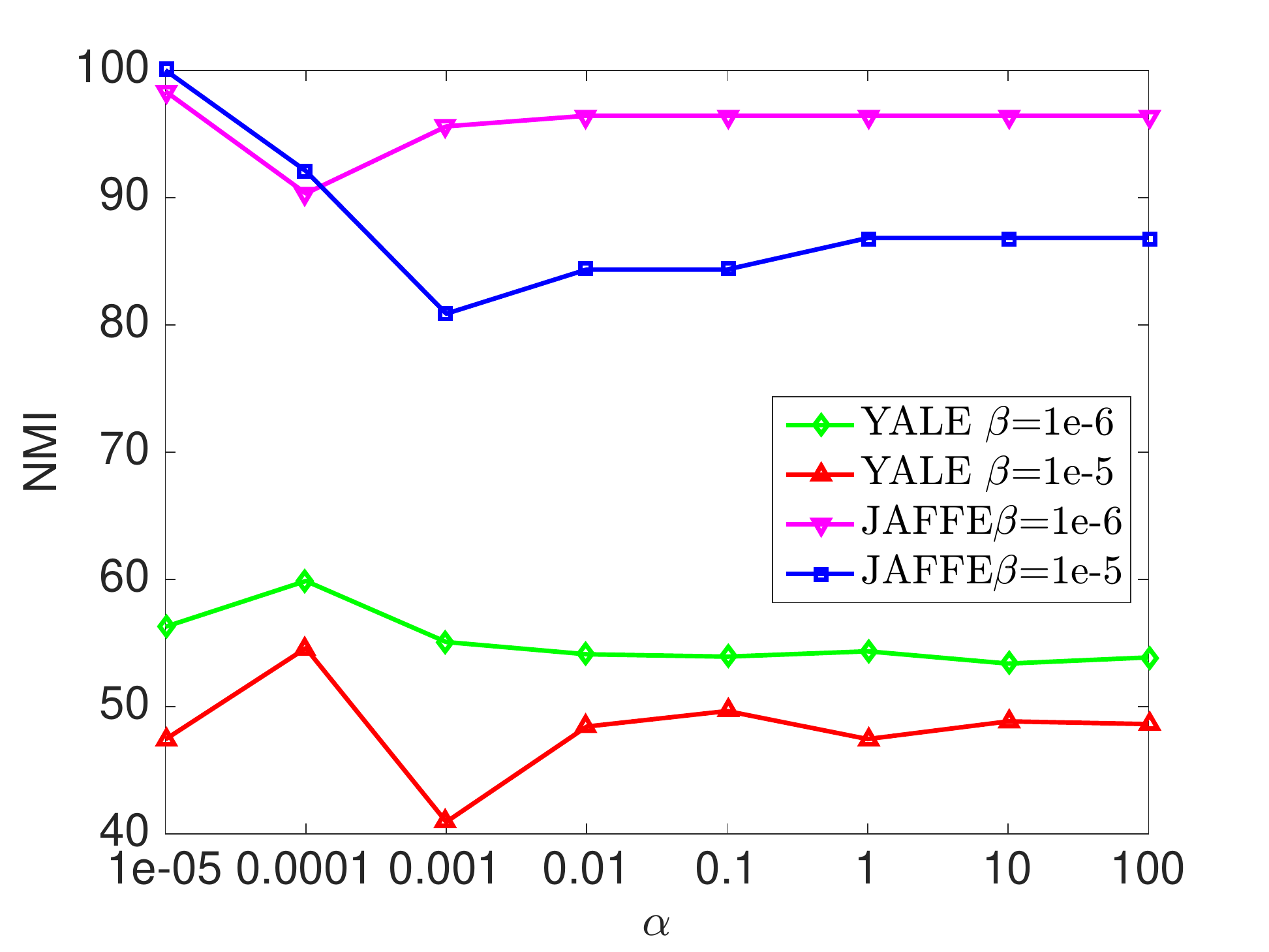}}
\subfloat[Purity\label{purity}]{\includegraphics[width=.3\textwidth]{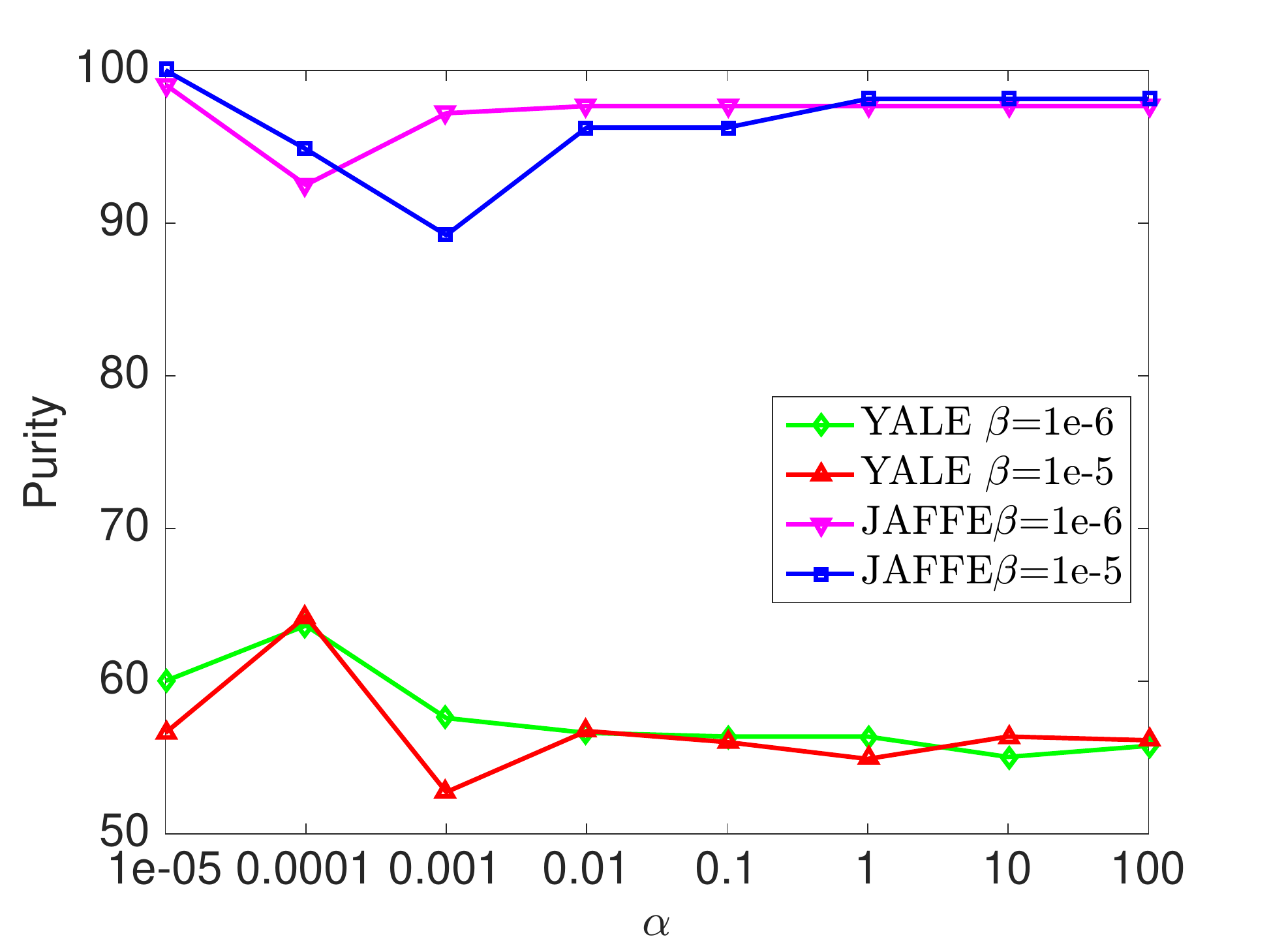}}
\caption{The effect of parameters $\alpha$ and $\beta$ on the YALE and JAFFE data sets.\label{para}}
\end{figure*}
In this section, we demonstrate the effectiveness of the proposed method on real world benchmark data sets. 
\subsection{Data Sets}
There are altogether eight benchmark data sets used in our experiments. Table \ref{data} summarizes the statistics of these data sets. Among them, five are image ones, and the other three are text corpora\footnote{http://www-users.cs.umn.edu/~han/data/tmdata.tar.gz}. The five image data sets consist of four commonly used face databases (ORL\footnote{http://www.cl.cam.ac.uk/research/dtg/attarchive/facedatabase.html}, YALE\footnote{http://vision.ucsd.edu/content/yale-face-database}, AR\footnote{http://www2.ece.ohio-state.edu/~aleix/ARdatabase.html} \cite{martinez2007ar} and JAFFE\footnote{http://www.kasrl.org/jaffe.html}), 
and a binary alpha digits data set BA\footnote{http://www.cs.nyu.edu/~roweis/data.html}. 

\subsection{Experiment Setup}
To assess the effectiveness of multiple kernel learning, we design 12 kernels which include: seven Gaussian kernels of the form $K(\vec{x},\vec{y})=exp(-\|\vec{x}-\vec{y}\|_2^2/(td_{max}^2))$, where $d_{max}$ is the maximal distance between samples and $t$ varies over the set $\{0.01, 0.05, 0.1, 1, 10, 50, 100\}$; a linear kernel $K(\vec{x},\vec{y})=\vec{x}^T\vec{y}$; four polynomial kernels $K(\vec{x},\vec{y})=(a+\vec{x}^T\vec{y})^b$ with $a=\{0,1\}$ and $b=\{2,4\}$. Furthermore, all kernels are rescaled to $[0,1]$ by dividing each element by the largest pair-wise squared distance.

For single kernel methods, we run kernel k-means (KKM) \cite{scholkopf1998nonlinear},  spectral clustering (SC) \cite{ng2002spectral}, robust kernel k-means (RKKM) \cite{du2015robust}, and our proposed SCSK\footnote{https://github.com/sckangz/AAAI17} on each kernel separately. The methods in comparison are downloaded from the their authors' websites. And we report both the best and the average results over all these kernels. 

For multiple kernel methods, we implement the following algorithms on a combination of above kernels.

MKKM\footnote{http://imp.iis.sinica.edu.tw/IVCLab/research/Sean/mkfc/code}. MKKM \cite{huang2012multiple} extends k-means in a multiple-kernel setting. However, it uses a different way to learn the kernel weight. 

AASC\footnote{http://imp.iis.sinica.edu.tw/IVCLab/research/Sean/aasc/code}. AASC \cite{huang2012affinity} extends spectral clustering to the situation where multiple affinities exist.

RMKKM\footnote{https://github.com/csliangdu/RMKKM}. RMKKM \cite{du2015robust} adopts $\ell_{21}$ norm to measure the loss of k-means. 

SCMK. Our proposed method for joint similarity learning and clustering with multiple kernels.  

For spectral clustering like SC and AASC, we run k-means on spectral embedding to obtain the clustering results. To reduce the influence of initialization, we follow the strategy suggested in \cite{yang2010image,du2015robust}, and we repeat clustering 20 times and present the results with the best objective values. We set the number of clusters to the true number of classes for all clustering algorithms.

To quantitatively evaluate the clustering performance, we adopt the three widely used metrics, accuracy (Acc), normalized mutual information (NMI) \cite{cai2009locality}, and Purity. 
\subsection{Clustering Result}
Table \ref{clusterres} shows the clustering results in terms of accuracy, NMI and Purity on all the data sets. It can be seen that the proposed SCSK and SCMK produce promising results. Especially, our method can substantially improve the performance on JAFFE, AR, BA, TR11, and TR45 data sets. The big difference between best and average results confirms the fact that the choice of kernel has a huge influence on the performance of single kernel methods. This difference motivates the development of multiple kernel learning method. Besides, multiple kernel clustering approaches usually improve the results over single kernel clustering methods. 
\subsection{Parameter Selection}
There are two parameters $\alpha$ and $\beta$ in our models. We let $\alpha$ vary over the range of $\{\textrm{1e-5, 1e-4, 1e-3, 0.01, 0.1, 1, 10, 100}\}$, and $\beta$ over $\{\textrm{1e-6, 1e-5}\}$. Figure \ref{para} shows how the clustering results of SCMK in terms of Acc, NMI, and Purity vary with $\alpha$ and $\beta$ on JAFFE and YALE data sets. We can observe that the performance of SCMK is very stable with respect to a large range of $\alpha$ values and it is more sensitive to the value of $\beta$. 

\section{Conclusion}
In this paper, we first propose a clustering method to simultaneously perform similarity learning and the cluster indicator matrix construction. In our method, the similarity learning and the cluster indicator learning are integrated within one framework; the method can be easily extended to kernel spaces, so as to capture nonlinear structure information. The connections of the proposed method to kernel k-means, k-means, and spectral clustering are also established. To avoid extensive search of the best kernel, we further incorporate multiple kernel learning into our model. Similarity learning, cluster indicator construction, and kernel weight learning can be boosted by using the results of the other two. Extensive experiments have been conducted on real-world benchmark data sets to demonstrate the superior performance of our method. 
\section{Acknowledgements}
This work is supported by US National Science Foundation Grants IIS 1218712. Q. Cheng is the corresponding author.  
\bibliographystyle{aaai}
\bibliography{ref}

\begin{thebibliography}{}

\bibitem[\protect\citeauthoryear{Cai \bgroup et al\mbox.\egroup
  }{2009}]{cai2009locality}
Cai, D.; He, X.; Wang, X.; Bao, H.; and Han, J.
\newblock 2009.
\newblock Locality preserving nonnegative matrix factorization.
\newblock In {\em IJCAI}, volume~9,  1010--1015.

\bibitem[\protect\citeauthoryear{Cai \bgroup et al\mbox.\egroup
  }{2013}]{cai2013heterogeneous}
Cai, X.; Nie, F.; Cai, W.; and Huang, H.
\newblock 2013.
\newblock Heterogeneous image features integration via multi-modal
  semi-supervised learning model.
\newblock In {\em Proceedings of the IEEE International Conference on Computer
  Vision},  1737--1744.

\bibitem[\protect\citeauthoryear{Du \bgroup et al\mbox.\egroup
  }{2015}]{du2015robust}
Du, L.; Zhou, P.; Shi, L.; Wang, H.; Fan, M.; Wang, W.; and Shen, Y.-D.
\newblock 2015.
\newblock Robust multiple kernel k-means using ℓ 2; 1-norm.
\newblock In {\em Proceedings of the 24th International Conference on
  Artificial Intelligence},  3476--3482.
\newblock AAAI Press.

\bibitem[\protect\citeauthoryear{Elhamifar and
  Vidal}{2009}]{elhamifar2009sparse}
Elhamifar, E., and Vidal, R.
\newblock 2009.
\newblock Sparse subspace clustering.
\newblock In {\em Computer Vision and Pattern Recognition, 2009. CVPR 2009.
  IEEE Conference on},  2790--2797.
\newblock IEEE.

\bibitem[\protect\citeauthoryear{Fan}{1949}]{fan1949theorem}
Fan, K.
\newblock 1949.
\newblock On a theorem of weyl concerning eigenvalues of linear transformations
  i.
\newblock {\em Proceedings of the National Academy of Sciences of the United
  States of America} 35(11):652.

\bibitem[\protect\citeauthoryear{Huang, Chuang, and
  Chen}{2012a}]{huang2012affinity}
Huang, H.-C.; Chuang, Y.-Y.; and Chen, C.-S.
\newblock 2012a.
\newblock Affinity aggregation for spectral clustering.
\newblock In {\em Computer Vision and Pattern Recognition (CVPR), 2012 IEEE
  Conference on},  773--780.
\newblock IEEE.

\bibitem[\protect\citeauthoryear{Huang, Chuang, and
  Chen}{2012b}]{huang2012multiple}
Huang, H.-C.; Chuang, Y.-Y.; and Chen, C.-S.
\newblock 2012b.
\newblock Multiple kernel fuzzy clustering.
\newblock {\em IEEE Transactions on Fuzzy Systems} 20(1):120--134.

\bibitem[\protect\citeauthoryear{Huang, Nie, and
  Huang}{2013}]{huang2013spectral}
Huang, J.; Nie, F.; and Huang, H.
\newblock 2013.
\newblock Spectral rotation versus k-means in spectral clustering.
\newblock In {\em AAAI}.

\bibitem[\protect\citeauthoryear{Huang, Nie, and Huang}{2015}]{huang2015new}
Huang, J.; Nie, F.; and Huang, H.
\newblock 2015.
\newblock A new simplex sparse learning model to measure data similarity for
  clustering.
\newblock In {\em Proceedings of the 24th International Conference on
  Artificial Intelligence},  3569--3575.
\newblock AAAI Press.

\bibitem[\protect\citeauthoryear{Johnson}{1967}]{johnson1967hierarchical}
Johnson, S.~C.
\newblock 1967.
\newblock Hierarchical clustering schemes.
\newblock {\em Psychometrika} 32(3):241--254.

\bibitem[\protect\citeauthoryear{Kang and Cheng}{2016}]{kang2016top}
Kang, Z., and Cheng, Q.
\newblock 2016.
\newblock Top-n recommendation with novel rank approximation.
\newblock In {\em 2016 SIAM Int. Conf. on Data Mining (SDM 2016)},  126--134.

\bibitem[\protect\citeauthoryear{Kang, Peng, and Cheng}{2015a}]{kang2015robust}
Kang, Z.; Peng, C.; and Cheng, Q.
\newblock 2015a.
\newblock Robust subspace clustering via smoothed rank approximation.
\newblock {\em IEEE Signal Processing Letters} 22(11):2088--2092.

\bibitem[\protect\citeauthoryear{Kang, Peng, and Cheng}{2015b}]{kang2015cikm}
Kang, Z.; Peng, C.; and Cheng, Q.
\newblock 2015b.
\newblock Robust subspace clustering via tighter rank approximation.
\newblock In {\em Proceedings of the 24th ACM International on Conference on
  Information and Knowledge Management},  393--401.
\newblock ACM.

\bibitem[\protect\citeauthoryear{Luo \bgroup et al\mbox.\egroup
  }{2011}]{luo2011multi}
Luo, D.; Nie, F.; Ding, C.; and Huang, H.
\newblock 2011.
\newblock Multi-subspace representation and discovery.
\newblock In {\em Joint European Conference on Machine Learning and Knowledge
  Discovery in Databases},  405--420.
\newblock Springer.

\bibitem[\protect\citeauthoryear{MacQueen}{1967}]{macqueen1967some}
MacQueen, J.
\newblock 1967.
\newblock Some methods for classification and analysis of multivariate
  observations.
\newblock In {\em Proceedings of the fifth Berkeley symposium on mathematical
  statistics and probability}, volume~1,  281--297.
\newblock Oakland, CA, USA.

\bibitem[\protect\citeauthoryear{Martinez and Benavente}{2007}]{martinez2007ar}
Martinez, A., and Benavente, R.
\newblock 2007.
\newblock The ar face database, 1998.
\newblock {\em Computer Vision Center, Technical Report} 3.

\bibitem[\protect\citeauthoryear{Mohar \bgroup et al\mbox.\egroup
  }{1991}]{mohar1991laplacian}
Mohar, B.; Alavi, Y.; Chartrand, G.; and Oellermann, O.
\newblock 1991.
\newblock The laplacian spectrum of graphs.
\newblock {\em Graph theory, combinatorics, and applications} 2(871-898):12.

\bibitem[\protect\citeauthoryear{Ng \bgroup et al\mbox.\egroup
  }{2002}]{ng2002spectral}
Ng, A.~Y.; Jordan, M.~I.; Weiss, Y.; et~al.
\newblock 2002.
\newblock On spectral clustering: Analysis and an algorithm.
\newblock {\em Advances in neural information processing systems} 2:849--856.

\bibitem[\protect\citeauthoryear{Nie \bgroup et al\mbox.\egroup
  }{2016}]{nie2016constrained}
Nie, F.; Wang, X.; Jordan, M.~I.; and Huang, H.
\newblock 2016.
\newblock The constrained laplacian rank algorithm for graph-based clustering.
\newblock In {\em Thirtieth AAAI Conference on Artificial Intelligence}.
\newblock Citeseer.

\bibitem[\protect\citeauthoryear{Nie, Wang, and
  Huang}{2014}]{nie2014clustering}
Nie, F.; Wang, X.; and Huang, H.
\newblock 2014.
\newblock Clustering and projected clustering with adaptive neighbors.
\newblock In {\em Proceedings of the 20th ACM SIGKDD international conference
  on Knowledge discovery and data mining},  977--986.
\newblock ACM.

\bibitem[\protect\citeauthoryear{Peng \bgroup et al\mbox.\egroup
  }{2015}]{peng2015subspace}
Peng, C.; Kang, Z.; Li, H.; and Cheng, Q.
\newblock 2015.
\newblock Subspace clustering using log-determinant rank approximation.
\newblock In {\em Proceedings of the 21th ACM SIGKDD International Conference
  on Knowledge Discovery and Data Mining},  925--934.
\newblock ACM.

\bibitem[\protect\citeauthoryear{Peng \bgroup et al\mbox.\egroup
  }{2016}]{pengnonnegative}
Peng, C.; Kang, Z.; Hu, Y.; Cheng, J.; and Cheng, Q.
\newblock 2016.
\newblock Nonnegative matrix factorization with integrated graph and feature
  learning.
\newblock {\em ACM Transactions on Intelligent Systems and Technology}.

\bibitem[\protect\citeauthoryear{Roweis and Saul}{2000}]{roweis2000nonlinear}
Roweis, S.~T., and Saul, L.~K.
\newblock 2000.
\newblock Nonlinear dimensionality reduction by locally linear embedding.
\newblock {\em Science} 290(5500):2323--2326.

\bibitem[\protect\citeauthoryear{Sch{\"o}lkopf, Smola, and
  M{\"u}ller}{1998}]{scholkopf1998nonlinear}
Sch{\"o}lkopf, B.; Smola, A.; and M{\"u}ller, K.-R.
\newblock 1998.
\newblock Nonlinear component analysis as a kernel eigenvalue problem.
\newblock {\em Neural computation} 10(5):1299--1319.

\bibitem[\protect\citeauthoryear{Wang \bgroup et al\mbox.\egroup
  }{2015}]{wang2015discriminative}
Wang, X.; Liu, Y.; Nie, F.; and Huang, H.
\newblock 2015.
\newblock Discriminative unsupervised dimensionality reduction.
\newblock In {\em Proceedings of the 24th International Conference on
  Artificial Intelligence},  3925--3931.
\newblock AAAI Press.

\bibitem[\protect\citeauthoryear{Yang \bgroup et al\mbox.\egroup
  }{2010}]{yang2010image}
Yang, Y.; Xu, D.; Nie, F.; Yan, S.; and Zhuang, Y.
\newblock 2010.
\newblock Image clustering using local discriminant models and global
  integration.
\newblock {\em IEEE Transactions on Image Processing} 19(10):2761--2773.

\bibitem[\protect\citeauthoryear{Yu \bgroup et al\mbox.\egroup
  }{2012}]{yu2012optimized}
Yu, S.; Tranchevent, L.; Liu, X.; Glanzel, W.; Suykens, J.~A.; De~Moor, B.; and
  Moreau, Y.
\newblock 2012.
\newblock Optimized data fusion for kernel k-means clustering.
\newblock {\em IEEE Transactions on Pattern Analysis and Machine Intelligence}
  34(5):1031--1039.

\bibitem[\protect\citeauthoryear{Zeng and Cheung}{2011}]{zeng2011feature}
Zeng, H., and Cheung, Y.-m.
\newblock 2011.
\newblock Feature selection and kernel learning for local learning-based
  clustering.
\newblock {\em IEEE transactions on pattern analysis and machine intelligence}
  33(8):1532--1547.

\bibitem[\protect\citeauthoryear{Zhang, Nie, and
  Xiang}{2010}]{zhang2010general}
Zhang, C.; Nie, F.; and Xiang, S.
\newblock 2010.
\newblock A general kernelization framework for learning algorithms based on
  kernel pca.
\newblock {\em Neurocomputing} 73(4):959--967.

\end{thebibliography}

\end{document}